\documentclass{article}

 \usepackage[main, final]{neurips_2026}

\usepackage[utf8]{inputenc} 
\usepackage[T1]{fontenc}    
\usepackage{hyperref}       
\usepackage{url}            
\usepackage{booktabs}       
\usepackage{multirow}       
\usepackage{graphicx}       
\usepackage{subcaption}     
\usepackage{amsfonts}       
\usepackage{amsmath}        
\usepackage{amsthm}
\newtheorem{proposition}{Proposition}
\newtheorem{remark}{Remark}
\usepackage{enumitem}
\usepackage{amssymb}        
\usepackage{nicefrac}       
\usepackage{microtype}      
\usepackage{xcolor}         

\title{Entropy-Gated Latent Recursion}

\author{
Soham Bhattacharjee\thanks{The first two authors contributed equally.}, Dushyant Singh Chauhan$^{*}$, Salem Lahlou, Martin Takac, and Nils Lukas
\\
Computing and Mathematical Sciences Division \\
Mohamed bin Zayed University of Artificial Intelligence
\\
Abu Dhabi, United Arab Emirates
\\
\texttt{sohambhattacharjeenghss@gmail.com}, \texttt{dushyant.chauhan@mbzuai.ac.ae}, \\
\texttt{Salem.Lahlou@mbzuai.ac.ae}, \texttt{Martin.Takac@mbzuai.ac.ae}, \texttt{nils.lukas@mbzuai.ac.ae}
}

\begin{document}

\maketitle

\begin{abstract}
Inference-time scaling has become the dominant lever for improving language-model reasoning, but existing methods often derive rollout diversity from a single source: stochastic token-level sampling. We identify a second, fully \emph{deterministic} and \emph{complementary} axis: the layer span $L$ at which a frozen model's top decoder layers are recursively re-applied at high-uncertainty tokens. Different choices of $L$ produce distinct rollouts that solve different subsets of problems, with no stochasticity. We instantiate this axis through \emph{Entropy-Gated Latent Recursion} (EGLR), a training-free decoding procedure that, combined with temperature sampling, turns a single-axis rollout pool into an $L\!\times\!T$ Cartesian sampling space at almost the same per-rollout cost. Across $8$ instruction-tuned models and $6$ math benchmarks, majority voting over the $L$ axis matches or improves on greedy decoding in nearly all settings and outperforms FLOP-matched beam search, and the gains carry over to a code benchmark. On MATH-500, the oracle ceiling of the joint pool exceeds that of either axis alone by $8$ to $10$ points. In a preliminary study, using the pool for GRPO rollouts with importance-sampling correction improves over standard GRPO. The code is available at \href{https://github.com/spotai-mbzuai/neurips26-eglr}{https://github.com/spotai-mbzuai/eglr}.
\end{abstract}

\section{Introduction}
\label{sec:intro}

\begin{figure}[!h]
    \centering
    \includegraphics[width=0.6\linewidth]{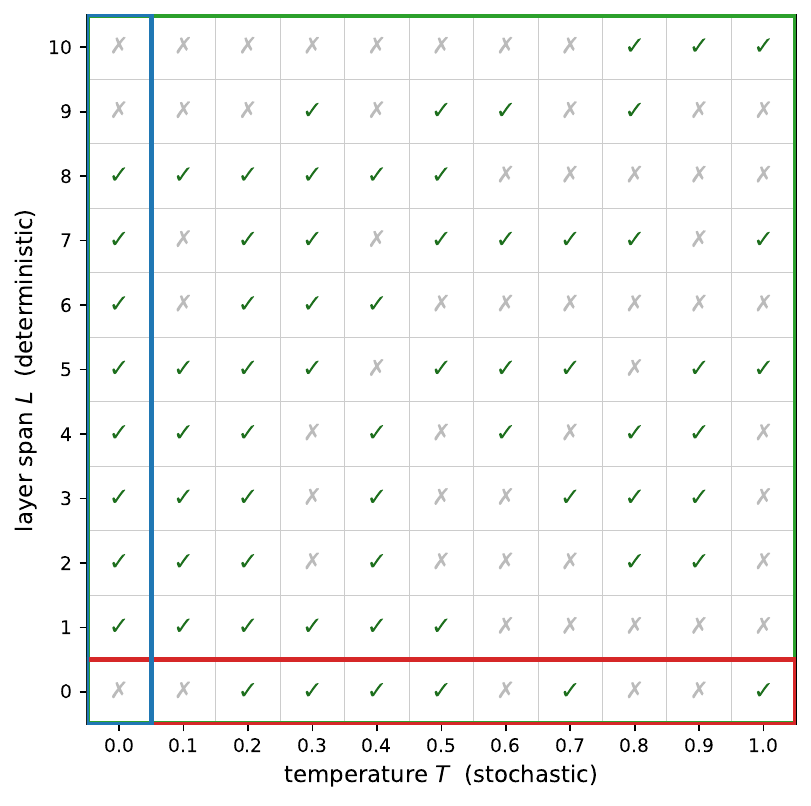}
    \caption{\textbf{The $L\!\times\!T$ sampling space} on a single MATH-500 problem (Qwen2.5-3B-Instruct, \#105). Each cell: correct (green \checkmark) or wrong (gray $\times$) under one $(L,T)$ configuration. \textcolor[HTML]{D62728}{\textbf{Red}}: $T$-only ($L\!=\!0$); \textcolor[HTML]{1F77B4}{\textbf{blue}}: $L$-only ($T\!=\!0$); \textcolor[HTML]{2CA02C}{\textbf{green}}: joint pool. Over all $500$ problems, the joint diversity ceiling is $91.6\%$ vs.\ $83.4\%$ ($T$-only) and $81.2\%$ ($L$-only).}
    \label{fig:teaser}
\end{figure}

Inference-time compute scaling~\citep{snell2024scaling, openai2024o1} has become a primary lever for improving reasoning without additional training. Self-consistency~\citep{wang2022sc}, best-of-$N$~\citep{cobbe2021grade, lightman2024prm}, and tree search~\citep{yao2023tot} all show that more inference compute yields better answers. Despite their differences, these methods share one structural property: \emph{trajectory diversity comes mainly from token-level stochasticity}, which makes their rollout space essentially one-dimensional, parameterized by temperature $T$.

Token-level stochasticity has limits: higher temperature adds noise rather than qualitatively different reasoning, and $M$ samples still yield $M$ draws from the same stochastic process~\citep{Li_2022}. This raises a structural question: \emph{is there a second, complementary axis along which a frozen model produces different rollouts without injecting noise?}

We study the \textbf{layer span $L$} as one such candidate. We instantiate it through \emph{Entropy-Gated Latent Recursion} (EGLR), a training-free procedure that re-applies the top-$L$ layers of a frozen model at high-uncertainty tokens for at most $K_{\max}$ iterations. Varying $L$ gives a fully \emph{deterministic} source of rollout diversity, and combining it with temperature turns the stochastic rollout pool into an \emph{$L\!\times\!T$ Cartesian sampling space} at almost the same per-rollout cost (Fig.~\ref{fig:teaser}). Latent recursion, entropy gating and looped computation each exist in prior work (Section~\ref{sec:related}); our contribution is to use the layer span of a frozen model as a deterministic axis of rollout diversity and to show that it complements temperature.

As a deployable aggregator over this pool, EGLR-SC matches or improves over greedy on $45/48$ (model, dataset) cells across $8$ instruction-tuned models and $6$ math reasoning benchmarks ($38$ strictly), and outperforms FLOP-matched beam search on $11/12$ cells. As a diagnostic of pool coverage, the \emph{diversity ceiling} (the oracle over a pool, which uses ground truth) of the joint $L\!\times\!T$ pool reaches $91.6\%$ on MATH-500 with Qwen2.5-3B-Instruct, $+8.2$ pp beyond the $T$-only ceiling ($83.4\%$) and $+10.4$ pp beyond the $L$-only ceiling ($81.2\%$); the joint pool has $100$ configurations against $10$ per axis.

Our contributions are: \textbf{(1)} a deterministic, training-free sampling axis, the layer span $L$ of a frozen model, complementary to temperature; \textbf{(2)} EGLR, an entropy-gated decoding method that iterates a frozen model's top-$L$ layers at high-uncertainty tokens, with no new parameters or fine-tuning; \textbf{(3)} the $L\!\times\!T$ rollout pool, characterized through per-cell accuracy, pairwise disagreement and ceiling decomposition; \textbf{(4)} EGLR-SC, a self-consistency aggregator over the pool; and \textbf{(5)} preliminary evidence beyond the main grid, on a code benchmark and as GRPO rollouts (Section~\ref{subsec:additional}).

\section{Related Work}
\label{sec:related}

\textbf{Inference-time compute scaling.}
Allocating additional compute at inference can rival gains from scaling model size~\citep{snell2024scaling}. Frontier systems such as o1~\citep{openai2024o1} and DeepSeek-R1~\citep{deepseek2025r1} achieve this through trained deliberation. EGLR operates in the complementary \emph{training-free} regime where compute is allocated dynamically by the model's own entropy signal, with no weight modification or auxiliary modules.

\textbf{Single-axis sampling and aggregation.}
Self-consistency~\citep{wang2022sc} and best-of-$N$ with reward models~\citep{cobbe2021grade,lightman2024prm} derive rollout diversity from temperature-induced stochasticity, so their rollout pool varies along a single axis. EGLR adds a second, fully \emph{deterministic} axis, the layer span $L$, which combines with temperature into an $L\!\times\!T$ pool that these aggregators can consume unchanged. Training-based alternatives instead fine-tune the generator for diverse reasoning, e.g.\ step-level GFlowNets guided by an automatically trained process reward model~\citep{younsi2025prmgfn}; EGLR obtains its second axis from a frozen model with no training.

\textbf{Latent and looped reasoning.}
Looped Transformers~\citep{giannou2023looped} re-apply the decoder block repeatedly, and later work studies looped models for reasoning~\citep{saunshi2025,ouro2025}; these methods train the recurrent structure. CoCoNuT~\citep{hao2024coconut} feeds the last hidden state back as the next input embedding, while HRM~\citep{wang2025hrm} and TRM~\citep{jolicoeur2025trm} reason through trained latent recursion on fixed grid tasks such as ARC and Sudoku. Quiet-STaR~\citep{zelikman2024quietstar} interleaves implicit reasoning during training. EGLR instead applies recursion at inference to a frozen, already-trained model, adding no parameters or training, and uses the recursion depth $L$ as a deterministic source of rollout diversity.

\textbf{Entropy-gated branching.}
Entropy-Gated Branching~\citep{li2026egb} shares EGLR's observation that a small subset of high-entropy tokens drives most of the prediction uncertainty, and expands those positions. The two methods differ in mechanism. EGB branches in \emph{token space}, generating multiple candidate continuations at uncertain positions and pruning them with an \emph{external feedback model}; EGLR recurses in \emph{latent layer space}, re-invoking the top-$L$ frozen layers to produce a \emph{deterministic} refined distribution with no auxiliary scorer.

\textbf{Contrastive, adaptive, and early-exit decoding.}
DoLa~\citep{chuang2024dola} and contrastive decoding~\citep{li2023contrastive} exploit layer-wise states within a single decoding pass, yielding one deterministic output with no sampling axis. Early-exit decoders~\citep{xin2020deebert,schuster2022confident} reduce compute on easy tokens; EGLR adds compute on uncertain ones without retraining. To our knowledge, prior training-free methods do not treat the layer span as a deterministic sampling axis that combines with temperature to form an $L\!\times\!T$ rollout pool for downstream aggregators.

\section{Method}
\label{sec:method}

\subsection{Preliminaries}
\label{subsec:prelim}


Let $f_\theta$ denote a frozen autoregressive Transformer language model with 
$N$ decoder layers, hidden dimension $d$, and vocabulary size $V$. For a context 
$x_{<t} = (x_1, \ldots, x_{t-1})$, the model produces hidden states layer by 
layer. Writing $h_t^{(0)}$ for the input embedding and $h_t^{(\ell)} \in 
\mathbb{R}^d$ for the residual-stream state \emph{at the output of layer $\ell$}, 
the forward pass obeys
\begin{equation}
    h_t^{(\ell)} \;=\; h_t^{(\ell-1)} 
    + \mathrm{Attn}^{(\ell)}\!\bigl(\mathrm{LN}^{(\ell)}_{\mathrm{a}}(h_t^{(\ell-1)})\bigr) 
    + \mathrm{MLP}^{(\ell)}\!\bigl(\mathrm{LN}^{(\ell)}_{\mathrm{m}}(\widetilde{h}_t^{(\ell)})\bigr),
    \label{eq:residual}
\end{equation}
for $\ell = 1, \ldots, N$, where $\widetilde{h}_t^{(\ell)}$ is the post-attention residual and $\mathrm{LN}$ denotes RMSNorm or LayerNorm. After the final layer, a normalization $\mathrm{LN}_{\mathrm{f}}$ and unembedding head $W_o \in \mathbb{R}^{V \times d}$ produce next-token logits and probabilities:
\begin{equation}
    z_t \;=\; W_o \, \mathrm{LN}_{\mathrm{f}}\!\bigl(h_t^{(N)}\bigr), 
    \qquad 
    p_t \;=\; \mathrm{softmax}(z_t) \in \Delta^{V-1}.
    \label{eq:logits}
\end{equation}
Greedy decoding emits $y_t = \arg\max_v \, p_{t,v}$. The token-level Shannon 
entropy will serve as our gating signal.
\begin{equation}
    H(p_t) \;=\; -\sum_{v=1}^{V} p_{t,v}\, \ln p_{t,v}
    \label{eq:entropy}
\end{equation}

\begin{figure}[t]
    \centering
    \includegraphics[width=\linewidth]{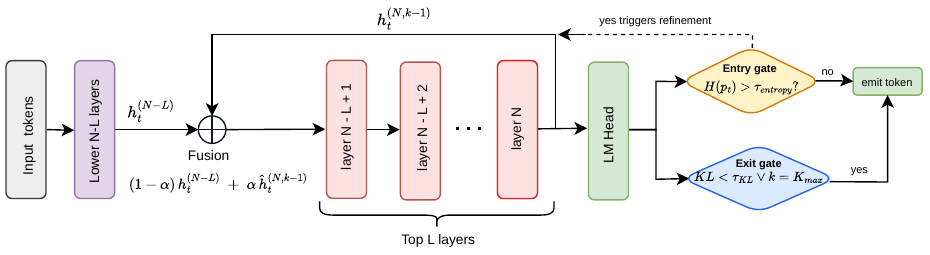}
    \caption{\textbf{Overview of Entropy-Gated Latent Recursion.}
At each decoding step, the entry gate checks $H(p_t)$ against $\tau_H$
(Eq.~\eqref{eq:gate}). Tokens below threshold are emitted greedily.
Above threshold, the top-$L$ layers are re-applied to the fused anchor
state (Eq.~\eqref{eq:fuse}) for up to $K_{\max}$ iterations, with
KL-based early exit (Eq.~\eqref{eq:kl_stop}).
Varying $L$ yields structurally distinct deterministic trajectories
(Section~\ref{subsec:diversity}).}
    \label{fig:model}
\end{figure}

\subsection{Entropy-Gated Compute Allocation}
\label{subsec:gate}

Token-level uncertainty is highly non-uniform: most reasoning tokens are near-deterministic, while a small minority carry the bulk of uncertainty. Allocating extra compute uniformly (as in beam search) is therefore wasteful. EGLR routes only tokens above a threshold $\tau_H$ to the refinement loop; all others decode greedily. Extra compute fires only when
\begin{equation}
    H(p_t) \;>\; \tau_H.
    \label{eq:gate}
\end{equation}
When Eq.~\eqref{eq:gate} fires the refinement of Section~\ref{subsec:recursion} runs; otherwise decoding proceeds greedily. $\tau_H$ is set automatically as the 95th-percentile per-token entropy of a greedy baseline run, capped at $2.5$ nats, targeting a trigger rate $\rho \approx 5\%$. The rule is label-free, since it uses only the model's own entropy distribution and never sees answers, and it is model and dataset specific because entropy scales differ across models. Overhead is analyzed in Section~\ref{subsec:cost}.

\subsection{Layer-wise Recursive Refinement}
\label{subsec:recursion}

Let $L$ denote the number of top decoder layers re-iterated, so the refinement zone is $\{N\!-\!L\!+\!1, \ldots, N\}$. Refinement operates on two anchor states from the original forward pass:
\begin{align}
    h_t^{(N-L)} &\;\text{is the output of layer $N-L$, i.e.\ input to the first 
    layer of the refinement zone,} \label{eq:anchor}\\
    h_t^{(N,0)} &\;=\; h_t^{(N)} \;\text{is the output of layer $N$ before 
    $\mathrm{LN}_{\mathrm{f}}$ is applied.} \label{eq:deep}
\end{align}
Both lie in the pre-norm residual-stream regime, ensuring statistically consistent fusion. Let $\mathcal{R}_L : \mathbb{R}^d \to \mathbb{R}^d$ apply layers $\{N-L+1, \ldots, N\}$ and return the output before $\mathrm{LN}_{\mathrm{f}}$:
\begin{equation}
    \mathcal{R}_L(h) \;\triangleq\; h^{(N)}\big|_{h^{(N-L)} = h}
    \label{eq:topL}
\end{equation}
Refinement proceeds for $k = 1, \ldots, K_{\max}$. First, $h_t^{(N,k-1)}$ is norm-matched to the anchor:
\begin{equation}
    \hat{h}_t^{(N,k-1)} \;=\; h_t^{(N,k-1)} 
    \cdot \frac{\bigl\|h_t^{(N-L)}\bigr\|}
               {\bigl\|h_t^{(N,k-1)}\bigr\|},
    \label{eq:normmatch}
\end{equation}
and the fused input, refined state, and updated distribution are then:
\begin{align}
    h_t^{\mathrm{in}, (k)} &\;=\; (1-\alpha)\, h_t^{(N-L)} 
        \;+\; \alpha\, \hat{h}_t^{(N,k-1)}, \label{eq:fuse}\\
    h_t^{(N,k)} &\;=\; \mathcal{R}_L\!\left(h_t^{\mathrm{in}, (k)}\right), 
        \label{eq:topL_apply}\\
    z_t^{(k)} &\;=\; W_o\, \mathrm{LN}_{\mathrm{f}}\!\bigl(h_t^{(N,k)}\bigr), 
        \qquad p_t^{(k)} \;=\; \mathrm{softmax}\!\bigl(z_t^{(k)}\bigr). 
        \label{eq:probs_k}
\end{align}
$\alpha \in [0, 1]$ controls how aggressively the refined state perturbs the anchor. Iteration terminates when
\begin{equation}
    \mathrm{KL}\!\left(p_t^{(k)} \,\|\, p_t^{(k-1)}\right) \;<\; \tau_{\mathrm{KL}},
    \label{eq:kl_stop}
\end{equation}
or the cap $K_{\max}$ is reached. The emitted token is $y_t = \arg\max_v p_{t,v}^{(k^\star)}$.

\subsection{Layer Span $L$ as a Complementary Sampling Axis}
\label{subsec:diversity}

Varying $L$ generates \emph{distinct, deterministic} reasoning trajectories.
By Eq.~\eqref{eq:residual}, anchor states at depths $L_1 \neq L_2$ differ by
the cumulative residual contributions of the intermediate layers, so the fused
input \eqref{eq:fuse} and the operator $\mathcal{R}_L$ in
\eqref{eq:topL_apply} both differ across configurations. This can flip the
argmax of $p_t^{\star}(L)$ at high-entropy positions, after which all
subsequent hidden states diverge. The process involves no stochasticity. For
fixed $L$ and greedy decoding, $p_t^\star(L)$ is a deterministic function of
$f_\theta$ and the prompt, so the diversity is \emph{structural} rather than
sampling noise. Proposition~\ref{prop:distinctness}
(Appendix~\ref{app:proof}) formalizes this under mild non-degeneracy
conditions.

\paragraph{Why $L$ is an axis.}
Temperature is also a tunable parameter, yet it is treated as a diversity axis
because sweeping it yields a family of rollouts rather than a single best
setting. $L$ meets the same criterion, and differs from temperature and top-$k$
in where it acts: those reshape probability mass over a distribution that has
already been computed, whereas $L$ changes the computation that produces the
distribution. Whether two axes are distinct is an empirical question, namely
whether they solve different problems, which we measure in
Section~\ref{subsec:diversity_deepdive}.

\paragraph{The $L\!\times\!T$ Cartesian sampling space.}
$L$ varies internal computation deterministically; $T$ varies token selection
stochastically. Their Cartesian product $\mathcal{T}_{L\times T} =
\{y(L,T)\}$ over $L \in \{1,\ldots,L_{\max}\}$ and $T \in
\{T_1,\ldots,T_n\}$ forms the inference-time sampling space studied in this
paper. Empirically (Section~\ref{subsec:diversity_deepdive}), the diversity
ceiling over $\mathcal{T}_{L\times T}$ exceeds either single-axis ceiling by
$+8.2$ pp and $+10.4$ pp on MATH-500 (Qwen2.5-3B), although the joint pool
contains more configurations than either single-axis pool.

\subsection{Aggregating Over the $L\!\times\!T$ Pool: EGLR-SC}
\label{subsec:eglrsc}

We instantiate the simplest aggregator over the $L\!\times\!T$ pool, \emph{majority-vote self-consistency} (EGLR-SC). Given a configuration set
\begin{equation}
    \mathcal{C} \;=\; \bigl\{(L_i,\, T_i)\bigr\}_{i=1}^{M},
\end{equation}
we run EGLR (greedy if $T_i = 0$, sampled otherwise) under each 
$(L_i, T_i)$, extract the final-answer string $a_i$ from each completion 
via the standard task-specific extractor (e.g.\ the boxed expression for 
math benchmarks), and aggregate via majority vote:
\begin{equation}
    \hat{a} \;=\; \arg\max_{a} \bigl|\{i : a_i = a\}\bigr|
    \label{eq:vote}
\end{equation}
Two regimes are of interest: (a)~purely deterministic ($T_i\!=\!0$ for all $i$), diversity from $L$ alone; (b)~Cartesian ($L$ layer spans $\times$ $T$ temperatures), exploiting both axes. As the budget $B$ grows, $\hat{a}$ inherits the concentration properties of self-consistency~\citep{wang2022sc} while drawing on a structurally richer rollout pool.

\paragraph{Diversity benefits beyond accuracy.}
The $L\!\times\!T$ pool can also serve procedures beyond voting. For GRPO-style RL~\citep{shao2024deepseekmath}, it supplies rollout groups with more varied candidates, which matters when every temperature sample for a prompt fails and near-identical rollouts weaken exploration~\citep{mishra2026sde2}. Because $L\!>\!0$ rollouts are off-policy with respect to the base model, training on them requires an importance-sampling correction (Section~\ref{subsec:additional}).



\subsection{Computational Cost}
\label{subsec:cost}

Let $C_{\mathrm{full}}$ denote the FLOP cost of a single full forward pass 
through $f_\theta$ across all $N$ layers, and let 
$\rho = P[H(p_t) > \tau_H]$ denote the refinement-trigger rate. The cost 
of one EGLR refinement iteration (Eqs.~\eqref{eq:fuse}--\eqref{eq:probs_k}) 
is dominated by the top-$L$ application in \eqref{eq:topL_apply}, which 
costs approximately $\frac{L}{N} C_{\mathrm{full}}$. The per-token cost of 
EGLR is bounded above by
\begin{equation}
    C_{\mathrm{EGLR}} \;\leq\; C_{\mathrm{full}} 
    \;+\; \rho \cdot K_{\max} \cdot \frac{L}{N}\, C_{\mathrm{full}} 
    \;=\; \Bigl(1 + \rho K_{\max}\,\tfrac{L}{N}\Bigr) C_{\mathrm{full}},
    \label{eq:cost}
\end{equation}
where the bound is tight only when every triggered token exhausts all $K_{\max}$ iterations; in practice KL early-exit reduces this below $K_{\max}$. For $\rho\!\approx\!0.05$, $K_{\max}\!=\!3$, $L\!=\!10$, $N\!=\!36$, worst-case overhead is below $5\%$ over greedy. EGLR-SC at $M$ configurations inherits a factor of $M$, matching standard self-consistency at $M$ samples. The $L$ axis adds no overhead beyond what one EGLR run already pays. Deployment implications appear in Appendix~\ref{app:practical}.

\section{Experiments}
\label{sec:experiments}

\subsection{Experimental Setup}
\label{subsec:setup}

\textbf{Models.}
We evaluate across eight frozen open-weight instruction-tuned models (0.5B--14B): Qwen2.5-\{0.5B,3B,7B,14B\}-Instruct~\citep{qwen2024qwen25}, Qwen2.5-Math-\{1.5B,7B\}-Instruct~\citep{yang2024qwen25math}, Llama-3.1-8B-Instruct~\citep{dubey2024llama3}, and Mistral-7B-Instruct-v0.3~\citep{jiang2023mistral7b}. No fine-tuning or auxiliary modules are used. Section~\ref{subsec:additional} adds Qwen3-4B-Instruct-2507 and a code benchmark.

\textbf{Datasets.}
We evaluate on six benchmarks spanning four orders of difficulty: \textbf{GSM8K}~\citep{cobbe2021grade} ($1{,}319$ grade-school problems); \textbf{MATH-500}~\citep{lightman2024prm,hendrycks2021math} ($500$ competition problems); \textbf{MinervaMath}~\citep{lewkowycz2022minerva} ($272$ undergraduate STEM problems); \textbf{AMC23}~\citep{aime_amc_dataset} ($40$ olympiad-prep problems); and \textbf{AIME24/25}~\citep{aime_amc_dataset} ($30$ problems each, hardest pre-Olympiad level).


\paragraph{Baselines.}
$L\!=\!0$ is standard greedy decoding (no refinement). We compare against \textbf{Greedy} ($1\times$ compute) and \textbf{Beam search}~at $B\!\in\!\{4,11\}$ on Qwen2.5-3B and 7B-Instruct; Beam-11 is FLOP-matched to the full $L\!=\!0,\ldots,10$ EGLR sweep (${\sim}11.2\!\times$ greedy). Temperature-based self-consistency~\citep{wang2022sc} at a matched budget of $B\!=\!10$ is compared on MATH-500 with Qwen2.5-3B (Table~\ref{tab:eglrsc_grid}).


\textbf{Metrics.}
We report exact-match \emph{accuracy}, \emph{oracle accuracy} (the fraction of problems solved by at least one configuration in a pool), and compute cost relative to greedy (Eq.~\eqref{eq:cost}), using each benchmark's canonical extraction pipeline. The oracle uses ground truth, so it is the performance ceiling an ideal selector could reach from a given pool; aggregation or sampling procedures such as EGLR-SC score at or below it.


\textbf{Hyperparameters.}
We use $\alpha\!=\!0.2$, $K_{\max}\!=\!3$, $\tau_{\mathrm{KL}}\!=\!5\!\times\!10^{-4}$ throughout. $\alpha$ and $K_{\max}$ were chosen from the ablation in Section~\ref{subsec:ablations} (MATH-500, Qwen2.5-3B-Instruct, $L\!=\!4$) and held fixed for every other model and benchmark. $\tau_H$ is set per (model, dataset) as the 95th-percentile per-token entropy of a greedy baseline run, capped at $2.5$ nats, yielding a trigger rate $\rho\!\approx\!5\%$; it is label-free. For EGLR-SC$_{L\times T}$: $L\!\in\!\{1,\ldots,10\}$, $T\!\in\!\{0.1,\ldots,1.0\}$, top-$k\!=\!50$ for $T\!>\!0$, max $3{,}072$ tokens, batch size $16$, seed $1{,}333$. Rows labelled ``best $L$'' select $L$ by test accuracy and are post hoc references; EGLR-SC does not select $L$.

\subsection{Results}
\label{subsec:results_singlecompletion}


Table~\ref{tab:main_singlecompletion} reports exact-match accuracy under four settings: \textbf{Greedy}; \textbf{EGLR (best $L$)}, the best single $L\!\in\!\{1,\ldots,10\}$ chosen on test accuracy, which is a post hoc reference; \textbf{EGLR-SC$_L$} (majority vote over $L\!\in\!\{1,\ldots,10\}$); and \textbf{EGLR-Oracle$_L$} (the oracle over the ten $L$ values, i.e.\ the ceiling an ideal selector could reach from the $L$-only pool). Table~\ref{tab:beam_comparison} reports the head-to-head against beam search on Qwen2.5-3B and 7B-Instruct at FLOP-matched compute. The results establish four findings.

\begin{table}[t]
\centering
\caption{Accuracy (\%) on six math reasoning benchmarks across eight models. EGLR-SC$_L$ ($B\!=\!10$) is a majority vote over $L\!=\!1,\ldots,10$ and matches or beats greedy on $45/48$ cells ($38$ strictly). EGLR (best $L$) is the best single $L\!\in\!\{1,\ldots,10\}$, chosen on test accuracy, so it can fall below greedy ($L\!=\!0$). EGLR-Oracle$_L$ is the fraction of problems solved by at least one $L\!\in\!\{1,\ldots,10\}$: a ceiling that an ideal selector could reach, computed with ground truth. Practical aggregators such as EGLR-SC score at or below it. No auxiliary models, training, or stochasticity.}
\label{tab:main_singlecompletion}
\scriptsize
\setlength{\tabcolsep}{3pt}
\resizebox{\linewidth}{!}{%
\begin{tabular}{llcccccc}
\toprule
\textbf{Model} & \textbf{Method} & \textbf{GSM8K} & \textbf{MATH-500} & \textbf{MinervaMath} & \textbf{AMC23} & \textbf{AIME24} & \textbf{AIME25} \\
\midrule
\multirow{4}{*}{\textbf{Qwen2.5-0.5B-Instruct}}      & \textbf{Greedy} & 40.5 & 22.0 & 2.6  & 10.0 & 0.0  & 0.0  \\
                                            & \textbf{EGLR (best $L$)} & 41.0 & 25.6 & 4.8  & 12.5 & 3.3  & 0.0  \\
                                            & \textbf{EGLR-SC$_L$ ($B\!=\!10$)} & 50.9 & 34.2 & 5.5  & 15.0 & 3.3  & 0.0  \\
                                            & \textbf{EGLR-Oracle$_L$} & 74.0 & 51.4 & 12.1 & 40.0 & 3.3  & 0.0  \\
\midrule
\multirow{4}{*}{\textbf{Qwen2.5-3B-Instruct}}        & \textbf{Greedy} & 83.1 & 64.6 & 15.4 & 37.5 & 6.7  & 3.3  \\
                                            & \textbf{EGLR (best $L$)} & 83.0 & 65.0 & 18.4 & 47.5 & 13.3 & 6.7  \\
                                            & \textbf{EGLR-SC$_L$ ($B\!=\!10$)} & 85.1 & 69.2 & 19.1 & 55.0 & 16.7 & 6.7  \\
                                            & \textbf{EGLR-Oracle$_L$} & 93.1 & 81.2 & 30.9 & 67.5 & 20.0 & 13.3 \\
\midrule
\multirow{4}{*}{\textbf{Qwen2.5-7B-Instruct}}        & \textbf{Greedy} & 89.1 & 73.2 & 22.1 & 52.5 & 6.7  & 10.0 \\
                                            & \textbf{EGLR (best $L$)} & 89.2 & 75.2 & 23.9 & 60.0 & 16.7 & 13.3 \\
                                            & \textbf{EGLR-SC$_L$ ($B\!=\!10$)} & 89.7 & 77.8 & 23.5 & 60.0 & 16.7 & 16.7 \\
                                            & \textbf{EGLR-Oracle$_L$} & 94.1 & 88.0 & 32.4 & 82.5 & 23.3 & 26.7 \\
\midrule
\multirow{4}{*}{\textbf{Qwen2.5-14B-Instruct}}       & \textbf{Greedy} & 92.4 & 77.0 & 27.2 & 67.5 & 10.0 & 16.7 \\
                                            & \textbf{EGLR (best $L$)} & 93.0 & 77.4 & 29.0 & 67.5 & 20.0 & 20.0 \\
                                            & \textbf{EGLR-SC$_L$ ($B\!=\!10$)} & 93.5 & 79.8 & 28.7 & 67.5 & 20.0 & 20.0 \\
                                            & \textbf{EGLR-Oracle$_L$} & 96.3 & 86.2 & 35.3 & 82.5 & 26.7 & 26.7 \\
\midrule
\multirow{4}{*}{\textbf{Qwen2.5-Math-1.5B-Instruct}} & \textbf{Greedy} & 83.3 & 70.4 & 19.5 & 55.0 & 13.3 & 13.3 \\
                                            & \textbf{EGLR (best $L$)} & 84.2 & 72.6 & 19.1 & 62.5 & 13.3 & 20.0 \\
                                            & \textbf{EGLR-SC$_L$ ($B\!=\!10$)} & 84.7 & 74.4 & 19.1 & 65.0 & 10.0 & 13.3 \\
                                            & \textbf{EGLR-Oracle$_L$} & 92.5 & 84.4 & 28.3 & 77.5 & 20.0 & 30.0 \\
\midrule
\multirow{4}{*}{\textbf{Qwen2.5-Math-7B-Instruct}}   & \textbf{Greedy} & 92.4 & 81.0 & 27.6 & 70.0 & 13.3 & 13.3 \\
                                            & \textbf{EGLR (best $L$)} & 92.9 & 81.2 & 26.8 & 67.5 & 20.0 & 13.3 \\
                                            & \textbf{EGLR-SC$_L$ ($B\!=\!10$)} & 92.8 & 82.8 & 27.9 & 62.5 & 13.3 & 13.3 \\
                                            & \textbf{EGLR-Oracle$_L$} & 94.5 & 88.2 & 32.0 & 82.5 & 30.0 & 26.7 \\
\midrule
\multirow{4}{*}{\textbf{Llama-3.1-8B-Instruct}}      & \textbf{Greedy} & 85.4 & 44.8 & 12.9 & 20.0 & 0.0  & 0.0  \\
                                            & \textbf{EGLR (best $L$)} & 85.0 & 47.4 & 16.9 & 30.0 & 13.3 & 3.3  \\
                                            & \textbf{EGLR-SC$_L$ ($B\!=\!10$)} & 87.9 & 58.0 & 18.0 & 47.5 & 10.0 & 6.7  \\
                                            & \textbf{EGLR-Oracle$_L$} & 94.8 & 72.8 & 27.9 & 50.0 & 23.3 & 6.7  \\
\midrule
\multirow{4}{*}{\textbf{Mistral-7B-Instruct-v0.3}}   & \textbf{Greedy} & 50.0 & 13.6 & 6.6  & 0.0  & 3.3  & 0.0  \\
                                            & \textbf{EGLR (best $L$)} & 49.8 & 14.8 & 8.1  & 12.5 & 3.3  & 0.0  \\
                                            & \textbf{EGLR-SC$_L$ ($B\!=\!10$)} & 56.6 & 18.8 & 9.2  & 2.5  & 3.3  & 0.0  \\
                                            & \textbf{EGLR-Oracle$_L$} & 77.3 & 36.2 & 17.6 & 25.0 & 3.3  & 0.0  \\
\bottomrule
\end{tabular}%
}
\end{table}

\textbf{(1) EGLR-SC$_L$ matches or improves over greedy on $45/48$ cells.} It is strictly better on $38$, ties on $7$, and falls below greedy on $3$ (Qwen2.5-Math-1.5B on MinervaMath and AIME24, and Qwen2.5-Math-7B on AMC23). On MATH-500, gains range from $+1.8$~pp (Qwen2.5-Math-7B) to $+13.2$~pp (Llama-3.1-8B), a training-free lift from the $L$ axis alone.

\textbf{(2) Oracle gains vary widely across models.} EGLR-Oracle$_L$ lifts Llama-3.1-8B from $44.8\%$ to $72.8\%$ ($+28.0$~pp) and Qwen2.5-0.5B from $22.0\%$ to $51.4\%$ ($+29.4$~pp) on MATH-500, while the math-specialized Qwen2.5-Math-7B gains $+7.2$~pp, to $88.2\%$. These values use ground truth, so they indicate how many problems the $L$ axis reaches, not deployable accuracy.

\textbf{(3) EGLR-SC$_L$ outperforms FLOP-matched beam search on $11$ of $12$ cells.} Across the cells in Table~\ref{tab:beam_comparison} (Qwen2.5-3B and 7B), EGLR-SC$_L$ beats Beam-11 on $11$, by $0.3$--$10.0$~pp. The exception (Qwen2.5-7B on AIME24, $16.7\%$ vs.\ $20.0\%$) is a one-problem difference on $30$ problems; on AIME25 the same model's Beam-11 drops to $6.7\%$ while EGLR-SC$_L$ reaches $16.7\%$.

\textbf{(4) The $L$-axis pool contains correct trajectories that beam search does not return.} EGLR-Oracle$_L$ is above Beam-11 on every cell, by $3.3$--$27.5$~pp (e.g., MATH-500 on Qwen2.5-3B: $81.2\%$ vs.\ $66.8\%$; AMC23 on Qwen2.5-7B: $82.5\%$ vs.\ $55.0\%$). This compares an oracle over ten rollouts with a single beam output, so it measures pool coverage rather than deployable accuracy; whether a selector can recover these solutions is left open (Section~\ref{sec:conclusion}).

\begin{table}[t]
\centering
\caption{EGLR-SC$_L$ vs.\ beam search at matched FLOPs. Beam-11 matches the full $L\!=\!0,\ldots,10$ sweep (${\sim}11.2\!\times$ greedy). EGLR-SC$_L$ wins $11/12$ cells; the exception is a one-problem difference on a 30-problem set. EGLR-Oracle$_L$ is the ceiling defined in Table~\ref{tab:main_singlecompletion}; it uses ground truth and is not a deployable result.}
\label{tab:beam_comparison}
\scriptsize
\setlength{\tabcolsep}{3pt}
\resizebox{\linewidth}{!}{%
\begin{tabular}{l l cccccc}
\toprule
\textbf{Model} & \textbf{Method} & \textbf{GSM8K} & \textbf{MATH-500} & \textbf{MinervaMath} & \textbf{AMC23} & \textbf{AIME24} & \textbf{AIME25} \\
\midrule
\multirow{6}{*}{\textbf{Qwen2.5-3B-Instruct}} & \textbf{Greedy} & 83.1 & 64.6 & 15.4 & 37.5 & 6.7  & 3.3  \\
                                     & \textbf{Beam-4} & 82.0 & 65.8 & 17.6 & 52.5 & 10.0 & 3.3  \\
                                     & \textbf{Beam-11} & 82.6 & 66.8 & 17.6 & 47.5 & 10.0 & 3.3  \\
                                     & \textbf{EGLR (best $L$)} & 83.0 & 65.0 & 18.4 & 47.5 & 13.3 & 6.7  \\
                                     & \textbf{EGLR-SC$_L$ ($B\!=\!10$)} & 85.1 & 69.2 & 19.1 & 55.0 & 16.7 & 6.7  \\
                                     & \textbf{EGLR-Oracle$_L$} & 93.1 & 81.2 & 30.9 & 67.5 & 20.0 & 13.3 \\
\midrule
\multirow{6}{*}{\textbf{Qwen2.5-7B-Instruct}} & \textbf{Greedy} & 89.1 & 73.2 & 22.1 & 52.5 & 6.7  & 10.0 \\
                                     & \textbf{Beam-4} & 88.8 & 74.4 & 22.8 & 52.5 & 16.7 & 13.3 \\
                                     & \textbf{Beam-11} & 88.3 & 75.8 & 23.2 & 55.0 & 20.0 & 6.7  \\
                                     & \textbf{EGLR (best $L$)} & 89.2 & 75.2 & 23.9 & 60.0 & 16.7 & 13.3 \\
                                     & \textbf{EGLR-SC$_L$ ($B\!=\!10$)} & 89.7 & 77.8 & 23.5 & 60.0 & 16.7 & 16.7 \\
                                     & \textbf{EGLR-Oracle$_L$} & 94.1 & 88.0 & 32.4 & 82.5 & 23.3 & 26.7 \\
\bottomrule
\end{tabular}%
}
\end{table}



\subsection{Sensitivity to Layer Span $L$}
\label{subsec:results_L}

Table~\ref{tab:L_sweep} (Appendix~\ref{app:per_l}) reports per-$L$ accuracy under greedy decoding for all eight models and six benchmarks. Three observations emerge. \textbf{(1)~No single $L$ is best across models.} On MATH-500 the best $L$ is $4$ for Qwen2.5-0.5B, $1$ for Qwen2.5-7B and Llama-3.1-8B, and $10$ for Qwen2.5-Math-7B, and it also varies across benchmarks for a given model. Individual $L$ values can fall below greedy (e.g., Qwen2.5-3B at $L\!=\!10$: $61.6\%$ vs.\ $64.6\%$), so the $L$ axis is useful as a pool to aggregate over, not as a single setting to tune. \textbf{(2)~Different $L$ values solve different problems.} The oracle over $L\!\in\!\{1,\ldots,10\}$ exceeds the best individual $L$ by $7$--$26$~pp on MATH-500 (e.g., Llama-3.1-8B: best $L$ $47.4\%$ vs.\ oracle $72.8\%$), so a substantial set of problems is solved by some $L$ but not by the best one. \textbf{(3)~The oracle exceeds greedy almost everywhere.} Across the $48$ (model, dataset) cells, the $L$-only oracle is above greedy in $45$; the three ties are at $0.0$--$3.3\%$ (AIME25 for Qwen2.5-0.5B and Mistral-7B, AIME24 for Mistral-7B). This includes the competition-level AIME sets, e.g., Qwen2.5-7B on AIME24 rises from $6.7\%$ to $23.3\%$.

\paragraph{Implication.}
These patterns suggest that $L$ does not behave like a compute knob for which more recursion is better. Here accuracy does not increase systematically with $L$ (Table~\ref{tab:L_sweep}), and more refinement iterations plateau after $K_{\max}\!=\!3$ (Table~\ref{tab:ablation}). A plausible interpretation is that each $L$ perturbs the hidden-state path of the frozen model differently, so different values lead to different decoding trajectories rather than successively better ones (Proposition~\ref{prop:distinctness}). This is why EGLR aggregates over $L$ instead of selecting a single value.

\subsection{Ablations}
\label{subsec:ablations}

Table~\ref{tab:ablation} ablates $\alpha$ and $K_{\max}$ on MATH-500 with Qwen2.5-3B ($L\!=\!4$); this is a single run on the evaluation set, and the defaults were chosen from it. Differences below about $1$~pp ($5$ problems) should not be over-read. We do not ablate $\tau_{\mathrm{KL}}$ or the percentile used for $\tau_H$. The percentile fixes the trigger rate ($\approx\!5\%$) and thus trades compute against refinement coverage.

\begin{table}[!h]
\centering
\caption{Ablation of the fusion weight $\alpha$ and refinement budget $K_{\max}$ on MATH-500 with Qwen2.5-3B-Instruct, $L=4$ (single run). Defaults are in bold.}
\label{tab:ablation}
\small
\setlength{\tabcolsep}{8pt}
\begin{tabular}{cc|cc}
\toprule
\multicolumn{2}{c|}{\textbf{Fusion weight $\alpha$}} & \multicolumn{2}{c}{\textbf{Max iterations $K_{\max}$}} \\
\textbf{Value} & \textbf{Acc.\ (\%)} & \textbf{Value} & \textbf{Acc.\ (\%)} \\
\midrule
0.1                    & 65.4 & 1                    & 63.2 \\
\textbf{0.2 (default)} & 65.0 & 2                    & 64.8 \\
0.5                    & 56.6 & \textbf{3 (default)} & 65.0 \\
0.7                    & 10.4 & 4                    & 64.8 \\
0.9                    &  0.6 & 5                    & 64.6 \\
\bottomrule
\end{tabular}
\end{table}

\textbf{(1) $K_{\max}$ plateaus from $2$ onward.} Accuracy rises from $K_{\max}\!=\!1$ ($63.2\%$) to $64.8$--$65.0\%$ for $K_{\max}\!\in\!\{2,3,4\}$ and stays at $64.6\%$ for $K_{\max}\!=\!5$. Larger caps add compute without a measurable gain, so we use $K_{\max}\!=\!3$.

\textbf{(2) $\alpha$ must remain small.} Accuracy is stable for $\alpha\!\in\!\{0.1, 0.2\}$ ($65.4\%$, $65.0\%$), drops sharply at $\alpha\!=\!0.5$ ($56.6\%$, $-8.4$~pp), and collapses for $\alpha\!\geq\!0.7$ ($10.4\%$ and $0.6\%$). A plausible explanation is that a large $\alpha$ lets the refined state dominate the anchor, so errors compound across iterations; $\alpha\!\in\!\{0.1,0.2\}$ is the safe range we tested.

EGLR therefore retains $L$, $\alpha$, $K_{\max}$ and $\tau_{\mathrm{KL}}$; $L$ is aggregated over rather than tuned, and $\tau_H$ follows a label-free percentile rule.

\subsection{Combining the $L$ and $T$ Axes}\label{subsec:diversity_deepdive}

We run the joint-axis study on one pair, MATH-500 with Qwen2.5-3B-Instruct, because the full $L\!\times\!T$ pool needs $100$ rollouts per problem. The three pools are $\mathcal{T}_L$ ($L\!\in\!\{1,\ldots,10\}$, greedy; $10$ configurations), $\mathcal{T}_T$ ($T\!\in\!\{0.1,\ldots,1.0\}$ at $L\!=\!0$, i.e.\ standard temperature sampling; $10$ configurations) and $\mathcal{T}_{L\times T}$ ($L\!\in\!\{1,\ldots,10\}$, $T\!\in\!\{0.1,\ldots,1.0\}$; $100$ configurations). Their oracle accuracies are $81.2\%$, $83.4\%$ and $91.6\%$ (Table~\ref{tab:eglrsc_grid}). Because the joint pool is ten times larger, its oracle advantage partly reflects pool size and should be read as a ceiling on diversity, not a matched-compute gain. At $B\!=\!10$ configurations, self-consistency reaches $74.2\%$ for the joint pool against $72.6\%$ for $T$ and $69.2\%$ for $L$ ($+1.6$~pp over $T$-only; a single run, about $8$ problems), while single configurations differ by under $2$~pp. Pairwise disagreement (Appendix~\ref{app:lxt_analysis}) shows the axes are not redundant: cross-axis disagreement (mean $78.3/500$) is higher than within-$L$ ($67.4$) and slightly lower than within-$T$ ($84.4$). The matched-budget comparison and the cell-level analysis in Section~\ref{subsec:quality} are the more conservative evidence. For procedures that consume per-prompt rollouts (RL group sampling, best-of-$N$ reranking), the joint pool offers more varied candidates at the same per-rollout cost.

\begin{table}[!h]
\centering
\caption{Accuracy (\%) for three rollout pools and three aggregation levels; $X$ in a row label is the column header ($T$, $L$ or $L\!\times\!T$). For $X\!=\!T$ no recursion is applied ($L\!=\!0$), i.e.\ standard temperature sampling. Pools have $10$ configurations ($T$, $L$; excluding greedy) or $100$ ($L\!\times\!T$). The oracle uses ground truth and is a ceiling; the other rows are deployable. Greedy ($L\!=\!0,T\!=\!0$) is the shared baseline.}
\label{tab:eglrsc_grid}
\small
\setlength{\tabcolsep}{8pt}
\begin{tabular}{l c c c}
\toprule
& $\boldsymbol{T}$ & $\boldsymbol{L}$ & $\boldsymbol{L\!\times\!T}$ \\
\midrule
\textbf{Greedy} ($L\!=\!0, T\!=\!0$) & \multicolumn{3}{c}{$64.6$} \\
\midrule
\textbf{EGLR (best $X$)}            & $66.0$ & $65.0$ & $66.2$ \\
\textbf{EGLR-SC$_X$ ($B\!=\!10$)}   & $72.6$ & $69.2$ & $74.2$ \\
\textbf{EGLR-Oracle$_X$}            & $83.4$ & $81.2$ & $91.6$ \\
\bottomrule
\end{tabular}
\end{table}


\subsection{Quality Analysis of the $L\!\times\!T$ Sampling Space}
\label{subsec:quality}


Beyond the aggregate ceilings of Section~\ref{subsec:diversity_deepdive}, we examine the cell-level structure of the same $11\!\times\!11$ joint grid (Figures~\ref{fig:grid_overview} and~\ref{fig:contribution}; pairwise disagreement in Appendix~\ref{app:lxt_analysis}, Figure~\ref{fig:axis_disagreement}).

\begin{figure}[!h]
\centering
\begin{subfigure}[t]{0.49\linewidth}
\centering
\includegraphics[width=\linewidth]{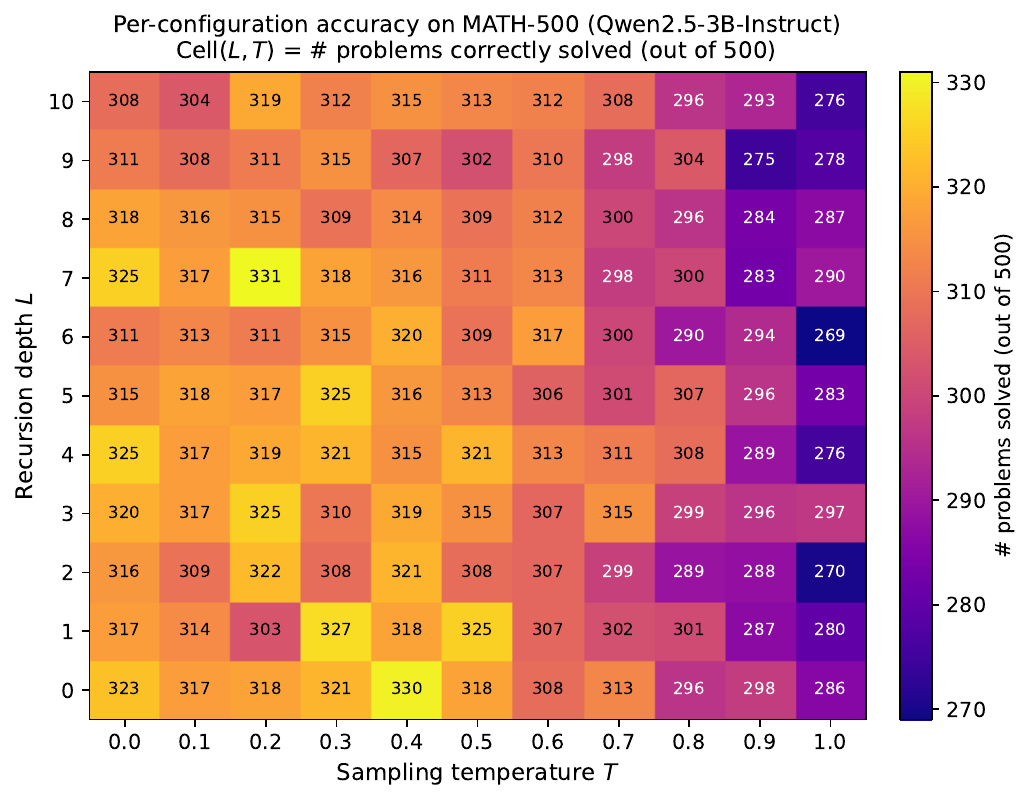}
\caption{Per-configuration accuracy (out of 500).}
\label{fig:per_config_acc}
\end{subfigure}\hfill
\begin{subfigure}[t]{0.49\linewidth}
\centering
\includegraphics[width=\linewidth]{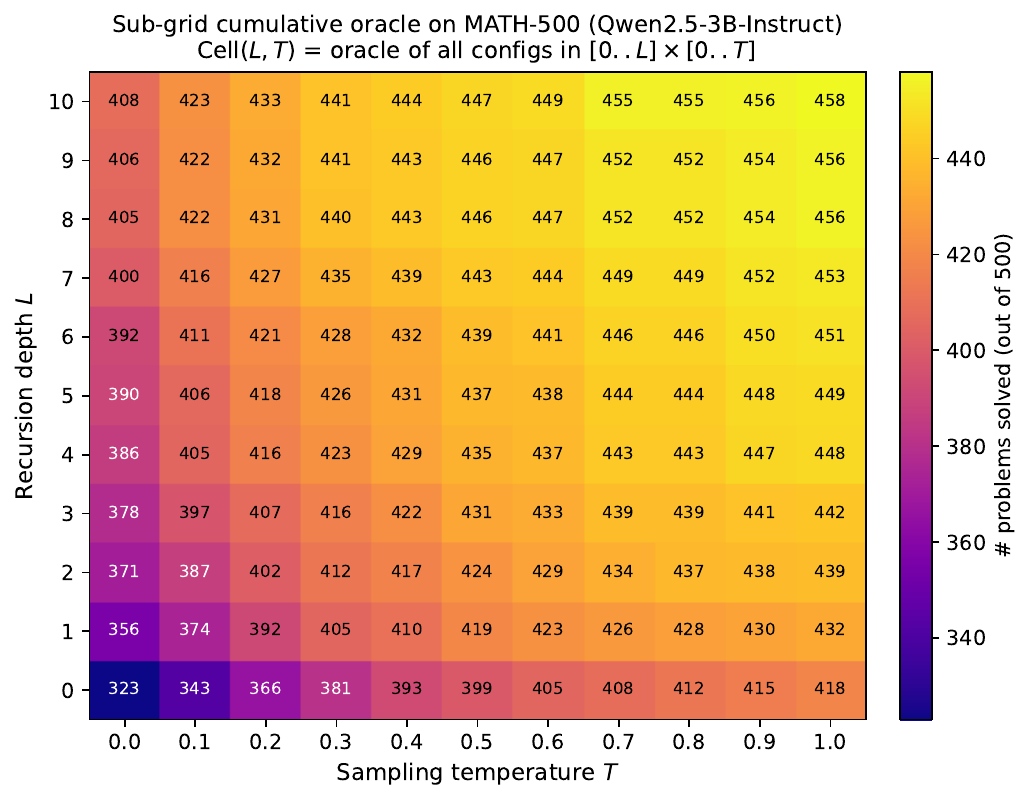}
\caption{Sub-grid cumulative oracle.}
\label{fig:oracle_heatmap}
\end{subfigure}
\caption{The $11\!\times\!11$ joint grid (MATH-500, Qwen2.5-3B-Instruct). Rows: $L\!\in\!\{0,\ldots,10\}$; columns: $T\!\in\!\{0.0,\ldots,1.0\}$. \textbf{(a)} Correct problems per $(L,T)$; no single cell dominates ($53.8$--$66.2\%$). Corners include the greedy cell, so they differ slightly from Table~\ref{tab:eglrsc_grid}.}
\label{fig:grid_overview}
\end{figure}


\textbf{Where in the space does correctness live?}
Figure~\ref{fig:contribution}(a) plots the \emph{exclusive} contribution: the problems that only that single $(L,T)$ configuration solves. Only $11$ of the $500$ problems are solved by exactly one configuration, spread over ten cells, so removing any one configuration costs at most $2$ problems and the ceiling does not rest on a few special cells. Panel~(b) plots the \emph{marginal-over-greedy} contribution: every non-greedy cell solves $20$--$44$ of the $177$ problems that greedy fails.

\begin{figure}[!h]
\centering
\includegraphics[width=\linewidth]{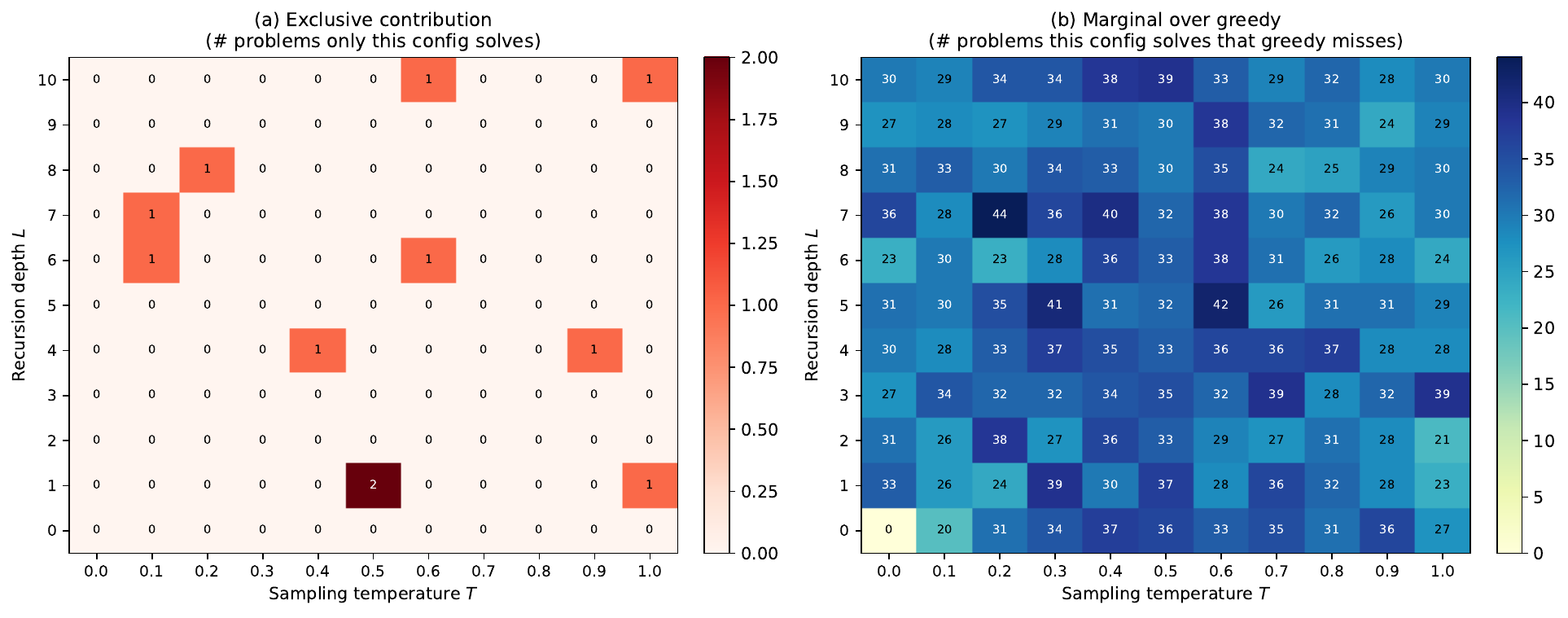}
\caption{Per-configuration contribution (same setup as Fig.~\ref{fig:grid_overview}). Cell$(L,T) = $ \# problems solved \textbf{(a)} \emph{only} by that config (almost all cells are zero), or \textbf{(b)} by that config but missed by greedy ($20$--$44$ per non-greedy cell).}
\label{fig:contribution}
\end{figure}

\subsection{Additional Evaluations}
\label{subsec:additional}

We test three extensions: a more recent model, a code benchmark, and RL rollouts (details in Appendix~\ref{app:additional}).

\paragraph{A more recent model.} We sweep $L\!\in\!\{0,\ldots,4\}$ on Qwen3-4B-Instruct-2507~\citep{qwen3techreport} and aggregate over $L\!\ge\!1$. On a fixed 250-problem MATH-500 subset at $T\!=\!0.5$, majority voting over $L$ raises accuracy from $83.6\pm0.3\%$ (single $L\!=\!0$ run) to $86.3\pm0.7\%$ ($+2.7$ pp; mean $\pm$ std over three seeds), and the oracle over the four spans reaches $90.4\pm0.3\%$ ($+6.8$ pp). This SC comparison pits a four-rollout pool against a single run, so it is not compute-matched. The oracle gain over $L\!=\!0$ grows with difficulty ($+1.4$, $+3.8$, $+5.5$, $+12.2$ pp at MATH levels 2--5; none at level 1).

\paragraph{A code benchmark.} On HumanEval+~\citep{humanevalplus} (164 problems, graded by unit tests; greedy), voting over $L\!\ge\!1$ raises pass@1 from $76.8\%$ to $79.9\%$ ($+3.1$ pp), and the oracle, which unit tests can realize as a verifier, reaches $82.9\%$ ($+6.1$ pp).

\paragraph{Rollouts for RL (preliminary).} We use the $L\!\times\!T$ pool as the rollout source for GRPO on Qwen3-4B. Because $L\!>\!0$ rollouts are off-policy, training uses a clipped importance-sampling ratio and an adaptive KL controller; evaluation uses plain decoding. In a single-seed, three-epoch run with identical untuned hyperparameters, accuracy improves over standard GRPO by $+1.1$ pp on GSM8K, $+3.5$ pp on MATH Precalculus and $+1.9$ pp on MATH Intermediate Algebra (Table~\ref{tab:grpo}). We treat this as a first indication, not an established training gain.

\section{Limitations}
\label{sec:limitations}

\textbf{Mixed gains for single configurations.}
A single $L$ without aggregation sometimes degrades accuracy (e.g., Qwen2.5-3B / MATH-500 at $L\!=\!10$: $61.6\%$ vs.\ greedy $64.6\%$); the gain comes from aggregating over the pool.
\textbf{Oracle is a ceiling.}
Oracle accuracies use ground truth, and majority voting recovers only about $28\%$ of the layer-only oracle gap on Qwen2.5-3B / MATH-500, so better aggregation is needed.
\textbf{Limited evaluation.}
The $L\!\times\!T$ analysis covers one model--dataset pair (a full characterization needs $100$ rollouts per problem), main-table results are single-seed, and the Qwen3, HumanEval+ and GRPO results cover one model, one code benchmark and one RL run.
\textbf{Early exit saturates.}
Mean recursion depth is close to $K_{\max}\!=\!3$ on most cells, and larger $K_{\max}$ gives no gain (Table~\ref{tab:ablation}).
\textbf{Deployment.}
Refinement is not batched in our implementation, so wall-clock overhead exceeds the FLOP bound ($1.13\times$ at $L\!=\!1$, $1.26\times$ at $L\!=\!4$ on Qwen3-4B), and EGLR is incompatible with paged-attention engines such as vLLM.

\section{Conclusion}
\label{sec:conclusion}
We introduced \emph{Entropy-Gated Latent Recursion} (EGLR), a training-free procedure that re-applies a frozen model's top-$L$ layers at high-uncertainty tokens. The layer span $L$ provides a deterministic sampling axis complementary to temperature, turning a one-axis stochastic rollout pool into an $L\!\times\!T$ Cartesian space at almost the same per-rollout FLOP cost. On MATH-500 with Qwen2.5-3B-Instruct, the oracle ceiling of the joint pool is $91.6\%$, $+8.2$ pp above the temperature-only pool. As a deployable aggregator, EGLR-SC matches or improves over greedy decoding on $45/48$ (model, dataset) cells and outperforms FLOP-matched beam search on $11/12$ cells, and initial results on a recent model, a code benchmark and GRPO rollouts suggest the pool is useful beyond the main grid.

\paragraph{Future work.}
First, EGLR-SC recovers only ${\sim}28\%$ of the layer-only oracle gap; because correlated errors across candidates can limit majority voting~\citep{liu2026martingale}, \emph{adaptive trajectory selection} via model-internal signals (entropy, KL speed, agreement across $L$) or a lightweight scorer is a natural next step. Second, the preliminary GRPO result motivates a multi-seed study of $L\!\times\!T$ rollouts for RL~\citep{shao2024deepseekmath}, where the deterministic $L$ axis gives a reproducible subset. Third, per-model oracle ceilings across all model--dataset pairs and a mechanistic study of fusion dynamics remain to be explored.

\section*{Acknowledgements and Disclosures}
\textbf{Funding.} The compute resources used for this work were provided by MBZUAI. The authors received no other third-party funding or support.

\textbf{Competing interests.} The authors declare no competing interests.

\bibliographystyle{plainnat}
\bibliography{custom}

\newpage
\appendix

\section{Practical Implications of the Cartesian Rollout Space}
\label{app:practical}

The $L\!\times\!T$ Cartesian rollout pool established in Section~\ref{subsec:eglrsc} has implications beyond inference-time accuracy alone. We highlight two illustrative settings.

\textbf{(a) Test-time best-of-$M$ at fixed compute.}
Where standard self-consistency must increase the sample count linearly to expand the candidate pool, the $L\!\times\!T$ space recovers a pool 
of $M_L$ layer configurations $\times$ $M_T$ temperature samples at the 
per-rollout cost given by Eq.~\eqref{eq:cost}. This is particularly 
impactful in the small-sample regime ($M\!=\!4$ or $M\!=\!8$), where 
stochastic-only ensembles concentrate on a small number of dominant 
trajectories and frequently miss alternative reasoning paths. The $L$ axis 
introduces no additional randomness and requires no extra generated tokens 
per rollout, so the expanded pool comes at the overhead already accounted 
for in Eq.~\eqref{eq:cost}. Crucially, because $L$ is a discrete 
hyperparameter while temperature is continuous, the two axes explore the 
rollout space in structurally distinct ways: varying $L$ moves through 
the model's hidden-state geometry deterministically, while varying 
temperature redistributes probability mass stochastically. The two axes 
are therefore non-redundant in practice, a property we confirm 
empirically in Section~\ref{sec:experiments}, where adding the $L$ axis 
to a fixed temperature ensemble raises the oracle ceiling by $8.2$ pp (Table~\ref{tab:eglrsc_grid})

\textbf{(b) Diverse rollout generation for outcome-supervised RL.}
Recent reasoning-focused reinforcement-learning pipelines such as 
GRPO~\citep{shao2024deepseekmath} estimate per-prompt advantages from a 
group of rollouts sampled at the same input, where rollout diversity 
directly governs the informativeness of the resulting gradient signal, and near-duplicate rollouts lead to repetitive, inefficient exploration~\citep{mishra2026sde2}. 
Standard implementations draw rollouts from temperature sampling alone, 
whose diversity is bounded by the model's stochastic output distribution. 
The $\mathcal{T}_{L \times T}$ Cartesian grid provides an $L\!\times\!T$ rollout pool of structurally distinct candidates at almost the same per-rollout cost, and our
empirical results (Section~\ref{sec:experiments}) show that this 
expansion yields non-redundant candidate answers rather than 
near-duplicate trajectories. Furthermore, the deterministic $L$ axis 
makes a subset of these rollouts exactly reproducible across optimization 
steps, a property that pure temperature sampling cannot provide. A preliminary GRPO study with importance-sampling correction appears in Section~\ref{subsec:additional}. A detailed multi-seed study of this connection remains future work.

\section{Per -$L$ Accuracy Analysis}
\label{app:per_l}

Table~\ref{tab:L_sweep} reports per-$L$ accuracy under greedy decoding across all eight models and all six benchmarks, alongside the cross-$L$ oracle (last column).

\textbf{(1) No single $L$ dominates.} The best individual $L$ differs across models and benchmarks: on MATH-500 alone, the best $L$ is $L{=}4$ for Qwen2.5-0.5B but $L{=}5$ for Qwen2.5-Math-1.5B, $L{=}1$ for Qwen2.5-7B and Llama-3.1-8B, $L{=}7$ for Qwen2.5-14B, and $L{=}10$ for Qwen2.5-Math-7B. There is no globally optimal layer span: each model and benchmark calls for a different $L$. This is precisely the structural-diversity property a self-consistency aggregator can exploit.

\textbf{(2) The $L$-axis sampling space contains correct trajectories not reachable from any single $L$ configuration.} Across the 48 (model, dataset) cells, the oracle column is strictly larger than the best individual $L$ accuracy in 44, often by $10$--$30$ percentage points (e.g., Llama-3.1-8B on MATH-500: best-$L$ $47.4\%$ vs.\ oracle $72.8\%$; Qwen2.5-0.5B on AMC23: best-$L$ $12.5\%$ vs.\ oracle $40.0\%$), and ties it in four AIME cells where accuracies are $0.0$--$3.3\%$.

\textbf{(3) The pattern is consistent across families and scales.} The $L$-axis oracle exceeds greedy decoding for every model family in our suite on nearly every benchmark, including the strict AIME competition sets where many models score in the single digits under greedy decoding.

\begin{table}[!h]
\centering
\caption{Per-$L$ accuracy (\%) under greedy decoding across all six benchmarks. \textbf{EGLR-Oracle$_L$}: union of correctly-solved problems across $L \in \{1, \ldots, 10\}$. The oracle gap (EGLR-Oracle$_L$ $-$ Greedy) characterizes the diversity available from varying the layer span alone.}
\label{tab:L_sweep}
\scriptsize
\setlength{\tabcolsep}{2.5pt}
\begin{tabular}{l c|cccccccccc|c}
\toprule
\textbf{Model} & \textbf{Greedy} & $\mathbf{L{=}1}$ & $\mathbf{L{=}2}$ & $\mathbf{L{=}3}$ & $\mathbf{L{=}4}$ & $\mathbf{L{=}5}$ & $\mathbf{L{=}6}$ & $\mathbf{L{=}7}$ & $\mathbf{L{=}8}$ & $\mathbf{L{=}9}$ & $\mathbf{L{=}10}$ & \textbf{EGLR-Oracle$_L$} \\
\midrule
\multicolumn{13}{c}{\textbf{MATH-500}} \\
\midrule
\textbf{Qwen2.5-0.5B-Instruct}      & 22.0 & 22.4 & 22.8 & 23.8 & 25.6 & 24.6 & 20.8 & 22.0 & 18.8 & 24.2 & 21.4 & \textbf{51.4} \\
\textbf{Qwen2.5-3B-Instruct}        & 64.6 & 63.4 & 63.2 & 64.0 & 65.0 & 63.0 & 62.2 & 65.0 & 63.6 & 62.2 & 61.6 & \textbf{81.2} \\
\textbf{Qwen2.5-7B-Instruct}        & 73.2 & 75.2 & 73.4 & 74.2 & 72.2 & 74.0 & 71.6 & 72.4 & 74.0 & 72.4 & 73.0 & \textbf{88.0} \\
\textbf{Qwen2.5-14B-Instruct}       & 77.0 & 76.4 & 77.2 & 75.2 & 76.6 & 76.0 & 76.8 & 77.4 & 76.8 & 76.2 & 76.6 & \textbf{86.2} \\
\textbf{Qwen2.5-Math-1.5B-Instruct} & 70.4 & 70.8 & 71.0 & 70.4 & 71.8 & 72.6 & 70.0 & 71.0 & 72.2 & 72.0 & 69.8 & \textbf{84.4} \\
\textbf{Qwen2.5-Math-7B-Instruct}   & 81.0 & 81.0 & 80.6 & 80.8 & 79.8 & 79.0 & 80.8 & 80.2 & 78.4 & 80.6 & 81.2 & \textbf{88.2} \\
\textbf{Llama-3.1-8B-Instruct}      & 44.8 & 47.4 & 44.0 & 46.2 & 45.8 & 42.8 & 43.2 & 44.6 & 45.4 & 43.6 & 43.4 & \textbf{72.8} \\
\textbf{Mistral-7B-Instruct-v0.3}   & 13.6 & 12.8 & 12.8 & 12.4 & 12.6 & 14.8 & 11.4 & 13.4 & 13.8 & 11.0 & 13.4 & \textbf{36.2} \\
\midrule
\multicolumn{13}{c}{\textbf{GSM8K}} \\
\midrule
\textbf{Qwen2.5-0.5B-Instruct}      & 40.5 & 41.0 & 37.8 & 40.0 & 39.9 & 38.6 & 37.8 & 37.8 & 36.5 & 37.1 & 37.9 & \textbf{74.0} \\
\textbf{Qwen2.5-3B-Instruct}        & 83.1 & 82.3 & 82.3 & 82.3 & 82.8 & 82.5 & 81.9 & 82.0 & 83.0 & 82.6 & 81.2 & \textbf{93.1} \\
\textbf{Qwen2.5-7B-Instruct}        & 89.1 & 88.2 & 89.2 & 88.9 & 88.4 & 88.6 & 88.2 & 87.8 & 88.2 & 88.8 & 88.3 & \textbf{94.1} \\
\textbf{Qwen2.5-14B-Instruct}       & 92.4 & 92.6 & 91.7 & 93.0 & 92.9 & 92.3 & 92.7 & 92.8 & 92.6 & 92.6 & 92.7 & \textbf{96.3} \\
\textbf{Qwen2.5-Math-1.5B-Instruct} & 83.3 & 82.4 & 84.2 & 83.8 & 83.3 & 83.2 & 82.4 & 84.1 & 83.6 & 81.5 & 82.4 & \textbf{92.5} \\
\textbf{Qwen2.5-Math-7B-Instruct}   & 92.4 & 92.5 & 92.5 & 92.5 & 92.1 & 91.9 & 92.5 & 92.2 & 92.3 & 92.3 & 92.9 & \textbf{94.5} \\
\textbf{Llama-3.1-8B-Instruct}      & 85.4 & 83.9 & 83.7 & 84.5 & 85.0 & 84.8 & 84.5 & 84.5 & 84.4 & 84.7 & 84.6 & \textbf{94.8} \\
\textbf{Mistral-7B-Instruct-v0.3}   & 50.0 & 48.2 & 47.5 & 49.0 & 48.0 & 48.7 & 46.7 & 47.5 & 49.1 & 49.8 & 47.4 & \textbf{77.3} \\
\midrule
\multicolumn{13}{c}{\textbf{MinervaMath}} \\
\midrule
\textbf{Qwen2.5-0.5B-Instruct}      & 2.6  & 3.7  & 3.3  & 2.2  & 4.8  & 1.8  & 3.7  & 2.2  & 4.0  & 4.8  & 3.7  & \textbf{12.1} \\
\textbf{Qwen2.5-3B-Instruct}        & 15.4 & 16.5 & 15.8 & 17.3 & 15.4 & 15.8 & 16.9 & 16.2 & 18.0 & 18.4 & 18.0 & \textbf{30.9} \\
\textbf{Qwen2.5-7B-Instruct}        & 22.1 & 21.7 & 20.2 & 21.0 & 22.4 & 21.0 & 23.9 & 21.0 & 22.8 & 21.0 & 18.8 & \textbf{32.4} \\
\textbf{Qwen2.5-14B-Instruct}       & 27.2 & 27.6 & 27.9 & 23.9 & 25.4 & 28.3 & 28.7 & 29.0 & 27.2 & 26.8 & 27.9 & \textbf{35.3} \\
\textbf{Qwen2.5-Math-1.5B-Instruct} & 19.5 & 19.1 & 19.1 & 17.6 & 18.8 & 18.8 & 17.3 & 18.8 & 17.6 & 17.6 & 17.6 & \textbf{28.3} \\
\textbf{Qwen2.5-Math-7B-Instruct}   & 27.6 & 26.1 & 26.8 & 24.6 & 26.8 & 25.7 & 25.4 & 25.4 & 24.3 & 25.4 & 24.3 & \textbf{32.0} \\
\textbf{Llama-3.1-8B-Instruct}      & 12.9 & 13.6 & 12.1 & 12.1 & 12.9 & 13.6 & 12.5 & 13.6 & 12.9 & 16.9 & 14.7 & \textbf{27.9} \\
\textbf{Mistral-7B-Instruct-v0.3}   & 6.6  & 5.9  & 7.0  & 6.2  & 6.2  & 4.4  & 8.1  & 5.5  & 4.0  & 5.1  & 6.2  & \textbf{17.6} \\
\midrule
\multicolumn{13}{c}{\textbf{AMC23}} \\
\midrule
\textbf{Qwen2.5-0.5B-Instruct}      & 10.0 & 10.0 & 12.5 & 12.5 & 2.5  & 10.0 & 12.5 & 5.0  & 12.5 & 10.0 & 7.5  & \textbf{40.0} \\
\textbf{Qwen2.5-3B-Instruct}        & 37.5 & 37.5 & 45.0 & 40.0 & 47.5 & 37.5 & 42.5 & 37.5 & 45.0 & 35.0 & 45.0 & \textbf{67.5} \\
\textbf{Qwen2.5-7B-Instruct}        & 52.5 & 60.0 & 52.5 & 47.5 & 45.0 & 50.0 & 60.0 & 55.0 & 50.0 & 50.0 & 52.5 & \textbf{82.5} \\
\textbf{Qwen2.5-14B-Instruct}       & 67.5 & 60.0 & 67.5 & 60.0 & 57.5 & 60.0 & 62.5 & 65.0 & 65.0 & 67.5 & 62.5 & \textbf{82.5} \\
\textbf{Qwen2.5-Math-1.5B-Instruct} & 55.0 & 62.5 & 50.0 & 45.0 & 47.5 & 60.0 & 62.5 & 52.5 & 57.5 & 50.0 & 55.0 & \textbf{77.5} \\
\textbf{Qwen2.5-Math-7B-Instruct}   & 70.0 & 62.5 & 57.5 & 62.5 & 67.5 & 60.0 & 60.0 & 55.0 & 57.5 & 62.5 & 52.5 & \textbf{82.5} \\
\textbf{Llama-3.1-8B-Instruct}      & 20.0 & 30.0 & 20.0 & 20.0 & 27.5 & 22.5 & 22.5 & 27.5 & 15.0 & 22.5 & 25.0 & \textbf{50.0} \\
\textbf{Mistral-7B-Instruct-v0.3}   & 0.0  & 0.0  & 7.5  & 0.0  & 2.5  & 0.0  & 0.0  & 7.5  & 2.5  & 12.5 & 2.5  & \textbf{25.0} \\
\midrule
\multicolumn{13}{c}{\textbf{AIME24}} \\
\midrule
\textbf{Qwen2.5-0.5B-Instruct}      & 0.0  & 3.3  & 0.0  & 0.0  & 0.0  & 0.0  & 0.0  & 0.0  & 0.0  & 0.0  & 0.0  & \textbf{3.3} \\
\textbf{Qwen2.5-3B-Instruct}        & 6.7  & 13.3 & 6.7  & 6.7  & 10.0 & 0.0  & 3.3  & 6.7  & 10.0 & 10.0 & 6.7  & \textbf{20.0} \\
\textbf{Qwen2.5-7B-Instruct}        & 6.7  & 13.3 & 13.3 & 10.0 & 10.0 & 13.3 & 6.7  & 3.3  & 13.3 & 13.3 & 16.7 & \textbf{23.3} \\
\textbf{Qwen2.5-14B-Instruct}       & 10.0 & 13.3 & 16.7 & 13.3 & 16.7 & 10.0 & 16.7 & 16.7 & 20.0 & 13.3 & 13.3 & \textbf{26.7} \\
\textbf{Qwen2.5-Math-1.5B-Instruct} & 13.3 & 10.0 & 10.0 & 13.3 & 6.7  & 10.0 & 10.0 & 6.7  & 10.0 & 10.0 & 6.7  & \textbf{20.0} \\
\textbf{Qwen2.5-Math-7B-Instruct}   & 13.3 & 10.0 & 6.7  & 6.7  & 13.3 & 13.3 & 13.3 & 20.0 & 16.7 & 16.7 & 13.3 & \textbf{30.0} \\
\textbf{Llama-3.1-8B-Instruct}      & 0.0  & 13.3 & 6.7  & 6.7  & 6.7  & 3.3  & 3.3  & 3.3  & 6.7  & 3.3  & 3.3  & \textbf{23.3} \\
\textbf{Mistral-7B-Instruct-v0.3}   & 3.3  & 0.0  & 0.0  & 0.0  & 0.0  & 0.0  & 0.0  & 0.0  & 0.0  & 3.3  & 0.0  & \textbf{3.3} \\
\midrule
\multicolumn{13}{c}{\textbf{AIME25}} \\
\midrule
\textbf{Qwen2.5-0.5B-Instruct}      & 0.0  & 0.0  & 0.0  & 0.0  & 0.0  & 0.0  & 0.0  & 0.0  & 0.0  & 0.0  & 0.0  & 0.0 \\
\textbf{Qwen2.5-3B-Instruct}        & 3.3  & 3.3  & 3.3  & 0.0  & 0.0  & 3.3  & 6.7  & 3.3  & 3.3  & 3.3  & 0.0  & \textbf{13.3} \\
\textbf{Qwen2.5-7B-Instruct}        & 10.0 & 10.0 & 6.7  & 10.0 & 3.3  & 10.0 & 0.0  & 13.3 & 0.0  & 6.7  & 3.3  & \textbf{26.7} \\
\textbf{Qwen2.5-14B-Instruct}       & 16.7 & 10.0 & 10.0 & 6.7  & 13.3 & 16.7 & 13.3 & 10.0 & 13.3 & 20.0 & 16.7 & \textbf{26.7} \\
\textbf{Qwen2.5-Math-1.5B-Instruct} & 13.3 & 10.0 & 10.0 & 6.7  & 10.0 & 13.3 & 20.0 & 10.0 & 20.0 & 20.0 & 3.3  & \textbf{30.0} \\
\textbf{Qwen2.5-Math-7B-Instruct}   & 13.3 & 6.7  & 10.0 & 13.3 & 13.3 & 10.0 & 10.0 & 10.0 & 13.3 & 13.3 & 13.3 & \textbf{26.7} \\
\textbf{Llama-3.1-8B-Instruct}      & 0.0  & 0.0  & 0.0  & 3.3  & 3.3  & 0.0  & 3.3  & 0.0  & 0.0  & 3.3  & 3.3  & \textbf{6.7} \\
\textbf{Mistral-7B-Instruct-v0.3}   & 0.0  & 0.0  & 0.0  & 0.0  & 0.0  & 0.0  & 0.0  & 0.0  & 0.0  & 0.0  & 0.0  & 0.0 \\
\bottomrule
\end{tabular}
\end{table}

\section{Full $L\!\times\!T$ Sampling Space Characterization}
\label{app:lxt_analysis}

\textbf{Are the axes complementary?}
The narrow accuracy band in Figure~\ref{fig:per_config_acc} could in 
principle be consistent with two very different scenarios: (i) configurations 
are roughly equivalent and solve essentially the same problems, or (ii) 
configurations have similar accuracy but solve \emph{different} problems. 
Only the latter implies real diversity. Figure~\ref{fig:axis_disagreement} 
discriminates between these by directly measuring pairwise disagreement, 
defined as the number of problems on which exactly one of two configurations 
is correct (i.e.\ the symmetric difference of their correct-sets). We slice the grid along three axes:
\begin{itemize}\itemsep=0pt
    \item \textbf{(a) $L$ vs $L$ (greedy, $T\!=\!0$).} Mean off-diagonal disagreement is $67.4$ out of $500$ ($13.5\%$), with maximum $81$.
    \item \textbf{(b) $T$ vs $T$ ($L\!=\!0$, no refinement).} Mean disagreement is $84.4$ ($16.9\%$), with extreme temperatures reaching $103$.
    \item \textbf{(c) $T$ vs $L$ (cross-axis).} Mean disagreement is $78.3$, with maximum $111$, meaningfully larger than the within-$L$ mean and slightly below the within-$T$ mean.
\end{itemize}
The cross-axis disagreement exceeds the within-$L$ value and is slightly below the within-$T$ value, so configurations on different axes differ about as often as configurations on the same axis, confirming that the two axes catch genuinely different problems.

\begin{figure}[!h]
\centering
\includegraphics[width=\linewidth]{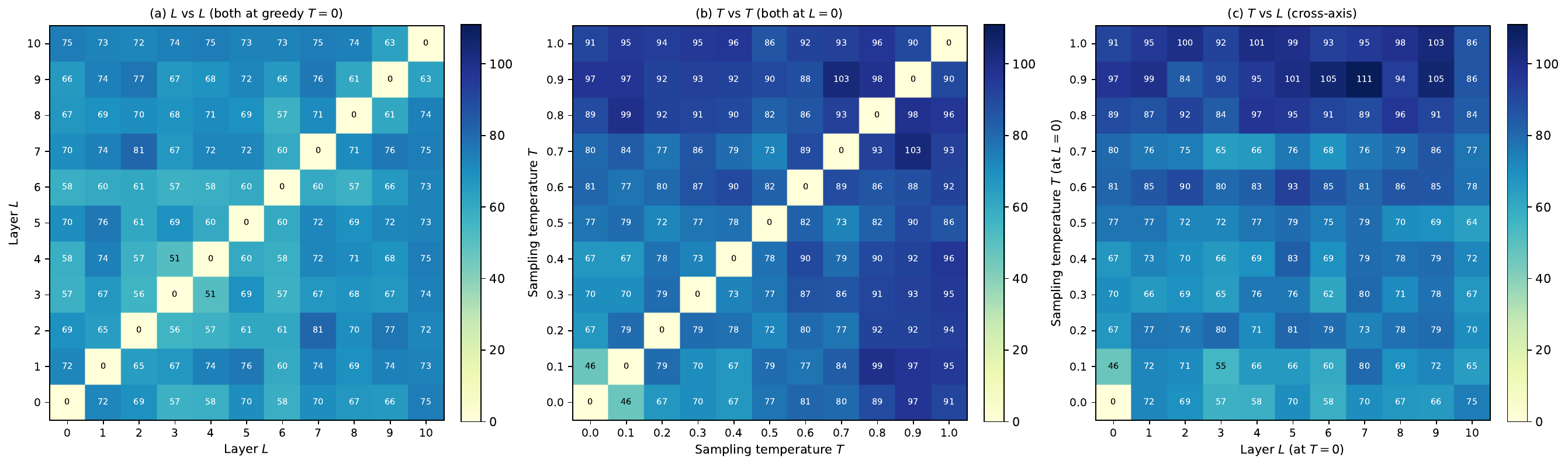}
\caption{Pairwise configuration disagreement on MATH-500 (Qwen2.5-3B-Instruct). Cell values count the problems on which exactly one of the two compared configurations is correct (out of 500); higher = more complementary. \textbf{(a) $L$ vs $L$:} pairs of layer spans at greedy $T\!=\!0$. \textbf{(b) $T$ vs $T$:} pairs of sampling temperatures at $L\!=\!0$. \textbf{(c) $T$ vs $L$:} cross-axis comparison. Cross-axis disagreements (c) are comparable to within-$T$ disagreements (b) and larger than within-$L$ disagreements (a).}
\label{fig:axis_disagreement}
\end{figure}

\section{Trajectory Distinctness Across Layer Spans}
\label{app:proof}

\begin{proposition}[Trajectory Distinctness]
\label{prop:distinctness}
Let $f_\theta$ be a frozen autoregressive transformer with $N$ decoder 
layers, and let $L_1 \neq L_2 \in \{1, \ldots, L_{\max}\}$ be two distinct 
layer configurations. Define the residual stream increment at layer $\ell$ 
as $\Delta_t^{(\ell)} = h_t^{(\ell)} - h_t^{(\ell-1)}$, and let 
$\mathcal{R}_{L} : \mathbb{R}^d \to \mathbb{R}^d$ denote the operator that 
applies layers $\{N-L+1, \ldots, N\}$ of $f_\theta$ to a hidden state and 
returns the pre-final-norm output. Suppose at least one of the following 
holds:
\begin{enumerate}[label=(\roman*)]
    \item \textbf{Non-degeneracy:} Assuming WLOG $L_1 > L_2$,
    $\displaystyle\sum_{\ell=N-L_1+1}^{N-L_2} \Delta_t^{(\ell)} \neq 
    \mathbf{0}$ for at least one token position $t$ at which refinement is triggered.
    \item \textbf{Operator distinguishability:} When the two fused inputs coincide, 
    $u_1 = u_2 = u$, the operators differ at that input, 
    $\mathcal{R}_{L_1}(u) \neq \mathcal{R}_{L_2}(u)$, at some triggered token position $t$.
\end{enumerate}
Then there exists at least one token position $t$ at which 
$p_t^{\star}(L_1) \neq p_t^{\star}(L_2)$, and the two EGLR generation 
trajectories are distributionally distinct.
\end{proposition}

\begin{proof}
Assume WLOG $L_1 > L_2$. All statements below hold generically, i.e.\ for 
non-degenerate weights, and concern a token position $t$ at which the entropy 
gate fires.

\textbf{Case (i).}
By assumption (i), the anchor states differ at some $t$:
\begin{equation}
    h_t^{(N-L_1)} - h_t^{(N-L_2)} 
    \;=\; \sum_{\ell=N-L_1+1}^{N-L_2} \Delta_t^{(\ell)} \;\neq\; \mathbf{0}.
    \label{eq:anchor_gap}
\end{equation}
Since the norm-matching factor $\|h_t^{(N-L)}\|/\|h_t^{(N)}\|$ depends on 
the anchor norm, it generically differs across $L_1$ and $L_2$ as well. Here 
$\hat{h}_t^{(N,0)}(L)$ denotes the norm-matched initial iterate, defined as
\begin{equation}
    \hat{h}_t^{(N,0)}(L) \;=\; h_t^{(N)} \cdot 
    \frac{\|h_t^{(N-L)}\|}{\|h_t^{(N)}\|},
\end{equation}
where $h_t^{(N)}$ is the output of the final decoder layer from the 
original forward pass, rescaled to match the L2 norm of the anchor 
$h_t^{(N-L)}$. Writing $u_1, u_2$ for the respective fused inputs at $k=1$,
\begin{align}
    u_i &= (1-\alpha)\,h_t^{(N-L_i)} + \alpha\,\hat{h}_t^{(N,0)}(L_i), 
    \quad i=1,2,
\end{align}
both terms differ, so generically $u_1 \neq u_2$ for $\alpha \in (0,1)$. The refined 
states are $\mathcal{R}_{L_1}(u_1)$ and $\mathcal{R}_{L_2}(u_2)$, which differ 
because the inputs differ and the two operators apply different sets of layers. 
Since $W_o$ generically has full column rank (hence is injective) and 
$\mathrm{LN}_\mathrm{f}$ and softmax preserve distinctness, distinct pre-norm 
outputs yield $p_t^\star(L_1) \neq p_t^\star(L_2)$. Further refinement iterations 
preserve this distinction generically.

\textbf{Case (ii).}
If assumption (i) fails then the anchors coincide at every position, hence so 
do their norms and the norm-matched iterates, and $u_1 = u_2 = u$. By 
assumption (ii), $\mathcal{R}_{L_1}(u) \neq \mathcal{R}_{L_2}(u)$ at some 
triggered position $t$, and the same chain through $W_o$, 
$\mathrm{LN}_\mathrm{f}$, and softmax yields $p_t^\star(L_1) \neq p_t^\star(L_2)$.

\textbf{Trajectory divergence.}
In both cases $p_t^\star(L_1) \neq p_t^\star(L_2)$ at some $t$. Whenever this 
changes the decoded token at $t$ (for greedy decoding, when it flips the argmax; 
under sampling, when it changes the emission probability of some token), the 
two generations emit different tokens at $t$, after which all subsequent 
residual-stream states diverge via Eq.~\eqref{eq:residual}.

\end{proof}

\begin{remark}
Assumption (i) holds whenever intermediate layers contribute non-trivially to the residual stream, which is the case for the trained models we evaluate. Assumption (ii) holds generically for sub-networks spanning different layer sets with non-degenerate weights. The distributional distinctness established above is consistent with what we observe in practice: distinct $L$ configurations produce different final answers on different problems (Table~\ref{tab:L_sweep} and Figure~\ref{fig:axis_disagreement}), as expected if $p_t^{\star}(L_1) \neq p_t^{\star}(L_2)$ leads to distinct deterministic generation trajectories.
\end{remark}

\section{Additional Experiments}
\label{app:additional}

\paragraph{Looped-transformer baseline.}
To test whether trained recursion could replace inference-time recursion, we adapted the core mechanisms of \citet{saunshi2025} to Qwen3-4B-Instruct-2507: hard weight-tying of the full model, and a cosine looping regularizer. Both were fine-tuned on the MATH training set (7{,}200 problems) and evaluated on the same 250-problem MATH-500 subset with a 3{,}072-token budget (Table~\ref{tab:looped}). Neither variant improved over the untrained base model. Hard looping (four shared layers in place of 36) collapsed entirely, and the regularized variant behaved like plain SFT. Plain SFT itself lowers accuracy sharply, which suggests that fine-tuning an already post-trained checkpoint on solution-only data erodes its reasoning behavior, so this comparison does not isolate the effect of recursion. We report it as evidence that adding recursion through training is not a drop-in alternative to inference-time recursion on a strong post-trained model.

\begin{table}[h]
\centering
\caption{Trained recursion versus EGLR on Qwen3-4B-Instruct-2507 (MATH-500, 250-problem subset, 3{,}072 tokens).}
\label{tab:looped}
\small
\begin{tabular}{lc}
\toprule
Method & Accuracy (\%) \\
\midrule
EGLR-SC$_L$ (ours, training-free) & 86.3 \\
Base model ($L\!=\!0$, frozen) & 83.6 \\
Plain SFT & 49.2 \\
Regularized SFT (cosine looping penalty) & 48.4 \\
Hard looping (weight-tied, $k\!=\!4$ layers) & 0.0 \\
\bottomrule
\end{tabular}
\end{table}

\paragraph{GRPO with $L\!\times\!T$ rollouts.}
We compare GRPO using rollouts from the $L\!\times\!T$ pool against standard temperature-only GRPO on Qwen3-4B. The two arms share all hyperparameters, none of which were tuned. Rollouts with $L\!>\!0$ are off-policy with respect to the base policy, so the policy-gradient loss uses a clipped importance ratio $\mathrm{clip}(\pi_{\text{base}}/\pi_{\text{EGLR}},\,1\!\pm\!\epsilon)$ and an adaptive KL controller; without this correction, an earlier run became unstable. Table~\ref{tab:grpo} reports accuracy under plain decoding. The \emph{rescue rate} is the fraction of prompt groups in which every $L\!=\!0$ rollout fails but at least one $L\!>\!0$ rollout succeeds, i.e.\ groups that would otherwise give no learning signal. It rises from $4.7\%$ on GSM8K to $9.4\%$ on Intermediate Algebra, and the accuracy gain follows the same pattern. Results are from a single seed and three epochs.

\begin{table}[h]
\centering
\caption{GRPO on Qwen3-4B with temperature-only rollouts (Standard) versus $L\!\times\!T$ rollouts (single seed, three epochs, untuned hyperparameters).}
\label{tab:grpo}
\small
\begin{tabular}{lcccc}
\toprule
Subject & Standard & $L\!\times\!T$ & $\Delta$ & Rescue rate \\
\midrule
GSM8K & 89.65 & 90.75 & $+1.1$ & 4.7\% \\
Precalculus & 22.16 & 25.64 & $+3.5$ & 6.7\% \\
Intermediate Algebra & 17.39 & 19.27 & $+1.9$ & 9.4\% \\
\bottomrule
\end{tabular}
\end{table}

\newpage
\section*{NeurIPS Paper Checklist}

\begin{enumerate}

\item {\bf Claims}
    \item[] Question: Do the main claims made in the abstract and introduction accurately reflect the paper's contributions and scope?
    \item[] Answer: \answerYes{}
    \item[] Justification: The abstract and Section~\ref{sec:intro} claim a deterministic layer-span axis $L$, complementary to temperature, that defines an $L\!\times\!T$ sampling space. Section~\ref{sec:experiments} supports this: EGLR-SC$_L$ matches or improves over greedy on $45/48$ (model, dataset) cells ($38$ strictly) and beats FLOP-matched beam search on $11/12$ cells; on MATH-500 with Qwen2.5-3B-Instruct the oracle ceiling of the joint pool is $91.6\%$, $+8.2$ pp over the $T$-only pool and $+10.4$ pp over the $L$-only pool. Oracle numbers are ceilings, not deployable accuracy. Section~\ref{subsec:additional} reports a recent model, a code benchmark, and a preliminary GRPO study.

\item {\bf Limitations}
    \item[] Question: Does the paper discuss the limitations of the work performed by the authors?
    \item[] Answer: \answerYes{}
    \item[] Justification: Section~\ref{sec:limitations} discusses: (i) single configurations can underperform greedy, so the gain rests on aggregation; (ii) the oracle is a ceiling and majority voting recovers only about $28\%$ of the layer-only oracle gap; (iii) the $L\!\times\!T$ analysis covers one model--dataset pair, main-table results are single-seed, and the Qwen3, HumanEval+ and GRPO results are limited in scope; (iv) KL early-exit saturates at $K_{\max}\!=\!3$; and (v) deployment, where refinement is unbatched and incompatible with paged-attention engines.

\item {\bf Theory assumptions and proofs}
    \item[] Question: For each theoretical result, does the paper provide the full set of assumptions and a complete (and correct) proof?
    \item[] Answer: \answerYes{}
    \item[] Justification: The trajectory-distinctness claim is stated as a proposition in Section~\ref{subsec:diversity} with all assumptions explicit; the full proof is given in Appendix~\ref{app:proof}.

    \item {\bf Experimental result reproducibility}
    \item[] Question: Does the paper fully disclose all the information needed to reproduce the main experimental results of the paper to the extent that it affects the main claims and/or conclusions of the paper (regardless of whether the code and data are provided or not)?
    \item[] Answer: \answerYes{}
    \item[] Justification: Section~\ref{subsec:setup} specifies the hyperparameters ($\alpha$, $K_{\max}$, $\tau_{\mathrm{KL}}$, the auto-calibration procedure for $\tau_H$), the models, the answer-extraction pipeline, batch size, maximum generation length, and seed; the method is formalised in Section~\ref{sec:method}, and the released code includes an example configuration.

\item {\bf Open access to data and code}
    \item[] Question: Does the paper provide open access to the data and code, with sufficient instructions to faithfully reproduce the main experimental results, as described in supplemental material?
    \item[] Answer: \answerYes{}
    \item[] Justification: Code with a README, example configuration and per-file documentation is released at \url{https://github.com/spotai-mbzuai/neurips26-eglr}; it reproduces the main experiments (Sections~\ref{subsec:setup}--\ref{subsec:quality}). All evaluated benchmarks (MATH-500, GSM8K, MinervaMath, AMC23, AIME24, AIME25, HumanEval+) are publicly available, and we use their official test splits.

\item {\bf Experimental setting/details}
    \item[] Question: Does the paper specify all the training and test details (e.g., data splits, hyperparameters, how they were chosen, type of optimizer) necessary to understand the results?
    \item[] Answer: \answerYes{}
    \item[] Justification: Test-time hyperparameters, dataset splits, batch sizes and the auto-calibration procedure for $\tau_H$ are reported in Section~\ref{subsec:setup}; $\alpha$ and $K_{\max}$ were chosen from the ablation in Section~\ref{subsec:ablations} on MATH-500, as stated there. EGLR itself is training-free; the GRPO setup is described in Appendix~\ref{app:additional} and the released code.

\item {\bf Experiment statistical significance}
    \item[] Question: Does the paper report error bars suitably and correctly defined or other appropriate information about the statistical significance of the experiments?
    \item[] Answer: \answerNo{}
    \item[] Justification: For computational reasons, main results are single-seed. The Qwen3-4B MATH-500 results in Section~\ref{subsec:additional} are reported as mean $\pm$ std over three seeds; HumanEval+ is greedy and deterministic; the GRPO result is single-seed. We note in Section~\ref{subsec:ablations} that differences below about $1$ pp ($5$ problems on MATH-500) should not be over-read, and EGLR-SC$_L$ matches or improves over greedy on $45/48$ cells.

\item {\bf Experiments compute resources}
    \item[] Question: For each experiment, does the paper provide sufficient information on the computer resources (type of compute workers, memory, time of execution) needed to reproduce the experiments?
    \item[] Answer: \answerYes{}
    \item[] Justification: Section~\ref{subsec:cost} gives the FLOP cost of EGLR (Eq.~\eqref{eq:cost}), and Section~\ref{sec:limitations} reports measured wall-clock overhead on Qwen3-4B ($1.13\times$ at $L\!=\!1$, $1.26\times$ at $L\!=\!4$ relative to $L\!=\!0$). Each (model, dataset, $L$) run used a single A100 GPU.

\item {\bf Code of ethics}
    \item[] Question: Does the research conducted in the paper conform, in every respect, with the NeurIPS Code of Ethics \url{https://neurips.cc/public/EthicsGuidelines}?
    \item[] Answer: \answerYes{}
    \item[] Justification: The work is a training-free inference-time method evaluated on publicly released models and benchmarks; it involves no human-subjects research, no new data collection, and no high-risk model release.

\item {\bf Broader impacts}
    \item[] Question: Does the paper discuss both potential positive societal impacts and negative societal impacts of the work performed?
    \item[] Answer: \answerYes{}
    \item[] Justification: Positive: by exposing a deterministic axis of inference-time diversity, EGLR can reduce the compute and energy required to obtain comparable diversity through stochastic sampling alone. Negative: the method inherits the same potential misuse profile as the underlying base LLMs, and we are not aware of any new risks introduced by the proposed inference procedure itself.

\item {\bf Safeguards}
    \item[] Question: Does the paper describe safeguards that have been put in place for responsible release of data or models that have a high risk for misuse (e.g., pre-trained language models, image generators, or scraped datasets)?
    \item[] Answer: \answerNA{}
    \item[] Justification: We release neither a new model nor a new dataset; the released artifact is decoding code applied to existing publicly-released models, which carries no additional misuse risk beyond those models themselves.

\item {\bf Licenses for existing assets}
    \item[] Question: Are the creators or original owners of assets (e.g., code, data, models), used in the paper, properly credited and are the license and terms of use explicitly mentioned and properly respected?
    \item[] Answer: \answerYes{}
    \item[] Justification: All evaluated models (Qwen2.5, Qwen2.5-Math, Qwen3, Llama-3.1, Mistral-7B) and benchmarks (MATH, GSM8K, MinervaMath, AMC, AIME, HumanEval+) are cited in the paper and used under their original public licenses on Hugging Face / their official sources.

\item {\bf New assets}
    \item[] Question: Are new assets introduced in the paper well documented and is the documentation provided alongside the assets?
    \item[] Answer: \answerYes{}
    \item[] Justification: The released code is the only new asset; it is accompanied by a README documenting the method, hyperparameters, file layout, an example config, usage examples, and the cost formula.

\item {\bf Crowdsourcing and research with human subjects}
    \item[] Question: For crowdsourcing experiments and research with human subjects, does the paper include the full text of instructions given to participants and screenshots, if applicable, as well as details about compensation (if any)?
    \item[] Answer: \answerNA{}
    \item[] Justification: The paper does not involve crowdsourcing or research with human subjects.

\item {\bf Institutional review board (IRB) approvals or equivalent for research with human subjects}
    \item[] Question: Does the paper describe potential risks incurred by study participants, whether such risks were disclosed to the subjects, and whether Institutional Review Board (IRB) approvals (or an equivalent approval/review based on the requirements of your country or institution) were obtained?
    \item[] Answer: \answerNA{}
    \item[] Justification: No human-subjects research is involved, so IRB approval is not applicable.

\item {\bf Declaration of LLM usage}
    \item[] Question: Does the paper describe the usage of LLMs if it is an important, original, or non-standard component of the core methods in this research? Note that if the LLM is used only for writing, editing, or formatting purposes and does \emph{not} impact the core methodology, scientific rigor, or originality of the research, declaration is not required.
    \item[] Answer: \answerNA{}
    \item[] Justification: LLMs are the \emph{subject} of study (the proposed inference-time procedure operates on a frozen base LLM), not a non-standard component of the methodology that would require declaration; any LLM usage was limited to writing/editing assistance, which the policy exempts.

\end{enumerate}

\end{document}